\documentclass{article}

\usepackage{microtype}
\usepackage{graphicx}
\usepackage{subfigure}
\usepackage{booktabs} % for professional tables

\usepackage{hyperref}

\usepackage[accepted]{icml2024}

\usepackage{amsmath}
\usepackage{amssymb}
\usepackage{mathtools}
\usepackage{amsthm}

\usepackage{thm-restate}

\usepackage[capitalize,noabbrev]{cleveref}

\theoremstyle{plain}
\newtheorem{theorem}{Theorem}[section]

\theoremstyle{definition}
\newtheorem{definition}[theorem]{Definition}

\theoremstyle{remark}

\usepackage[textsize=tiny]{todonotes}

\newenvironment{hproof}{%
  \proof}{\endproof}

\icmltitlerunning{Implicit Regularization in Feedback Alignment}

\begin{document}

\twocolumn[
\icmltitle{Implicit Regularization in Feedback Alignment \\ Learning Mechanisms for Neural Networks}
%\icmltitle{Aligning the Dots: Unraveling Feedback Alignment Mechanisms in Neural Networks}



\icmlsetsymbol{equal}{*}

\begin{icmlauthorlist}
\icmlauthor{Zachary Robertson}{yyy}
\icmlauthor{Oluwasanmi Koyejo}{yyy}
\end{icmlauthorlist}

\icmlaffiliation{yyy}{Department of Computer Science, Stanford, California, United States}
%\icmlaffiliation{comp}{Company Name, Location, Country}
%\icmlaffiliation{sch}{School of ZZZ, Institute of WWW, Location, Country}

\icmlcorrespondingauthor{Zachary Robertson}{zroberts@stanford.edu}
%\icmlcorrespondingauthor{Firstname2 Lastname2}{first2.last2@www.uk}

\icmlkeywords{Deep Learning, Implicit Regularization, Neural Networks, Bio-Plausible, ICML}

\vskip 0.3in
]



\printAffiliationsAndNotice{}  % leave blank if no need to mention equal contribution

\begin{abstract}

Feedback Alignment (FA) methods are biologically inspired local learning rules for training neural networks with reduced communication between layers. While FA has potential applications in distributed and privacy-aware ML, limitations in multi-class classification and lack of theoretical understanding of the alignment mechanism have constrained its impact. This study introduces a unified framework elucidating the operational principles behind alignment in FA. Our key contributions include: (1) a novel conservation law linking changes in synaptic weights to implicit regularization that maintains alignment with the gradient, with support from experiments, (2) sufficient conditions for convergence based on the concept of alignment dominance, and (3) empirical analysis showing better alignment can enhance FA performance on complex multi-class tasks. Overall, these theoretical and practical advancements improve interpretability of bio-plausible learning rules and provide groundwork for developing enhanced FA algorithms.

\end{abstract}

\section{Introduction}

Artificial neural networks (ANNs), inspired by biological neural networks, have revolutionized machine learning  \citet{mcculloch1943logical}. They were further developed in the connectionist framework, an approach to study cognition through the use of ANNs \cite{elman1996rethinking, medler1998brief}. This approach continues to demonstrate utility in modeling neural processes \cite{yamins2016using, yildirim2019integrative, richards2019deep, peters2021capturing}. Despite their effectiveness, many works highlight the non-local nature of backpropagation, which contradicts the local processing observed in biological neurons \cite{crick1989recent, lillicrap2016random, cornford2021learning}.

A key issue is the \textit{weight-transport problem}, where forward synaptic weights must equal feedback propagation weights during training \cite{grossberg1987competitive, lillicrap2016random}. Another issue is the contradiction of \textit{Dale's principle} which states neuron synaptic weights tend to stay either positive or negative \cite{eccles1976electrical, cornford2021learning}.

This work focuses on a promising alternative to backpropagation, the Feedback Alignment (FA) method, that circumvents the weight-transport problem by employing fixed, random feedback weights \cite{lillicrap2016random}. These methods have offered practical advantages in distributed, federated, and differentially private machine learning \cite{launay2020direct, lee2020differentially,jung2023lafd}. Despite its potential, FA's performance lags in multi-class classification tasks, and its alignment mechanism is not fully understood \cite{liao2016important, moskovitz2018feedback, bartunov2018assessing}. In this work we propose a novel theoretical framework that demystifies the alignment mechanism in FA, focusing on how alignment with the gradient can emerge. This framework not only sheds light on FA's operational principles but also accounts for its limitations in complex classification scenarios, marking a significant step towards more biologically plausible learning models. Our specific contributions are:

\begin{enumerate}
    \item \textbf{Theoretical Framework:} Introduction of a modular framework formalizing alignment and convergence properties to demystify the mechanisms behind FA, bridging gaps in understanding.
    \item \textbf{Implicit Regularization:} A conservation law linking changes in synaptic weights to regularization that inherently maintains feedback alignment, explaining the emergence of effects like Dale's principle.
    \item \textbf{Enhanced Alignment and Performance:} Empirical analysis showing techniques encouraging greater alignment, such as proper weight initialization, can enhance FA performance on complex multi-class tasks like CIFAR-100 and Tiny-ImageNet.
\end{enumerate}

We face several technical challenges establishing our results. A key novelty is that we identify for the first time a general implicit regularization of non-gradient based training methods. Critically, we leverage this finding to directly address the weight-transport problem, illustrating that effective learning can occur without mirrored synaptic weights. This allows us to identify initialization strategies for the feedback weights that allow us to connect alignment to convergence. Overall, addressing these technical challenges offers insights into how neural networks can maintain biological plausibility while achieving computational efficiency.

The remainder of this paper is organized as follows: Section \ref{related} discusses related work. Section \ref{prelim} provides a detailed background on learning models, the bio-plausibility problem, and other technical aspects we need to develop our results. Section \ref{framework} introduces our theoretical framework and its basic principles. Section \ref{empirical} discusses our empirical methodology and experimental results.  Finally, Section \ref{conclusion} concludes the paper with a summary of our contributions and suggestions for future research directions.

\section{Related Work}
\label{related}

\textbf{Motivating Bio-Plausible Learning:} The pursuit of bio-plausible learning mechanisms extends beyond mere academic interest, offering practical advantages in distributed, federated, and differentially private  machine learning \cite{launay2020direct, lee2020differentially,jung2023lafd}. \citet{launay2020direct} show FA methods can scale to modern deep learning scenarios and argue these methods can reduce communication overheads. \citet{lee2020differentially} show FA may be more suitable for differential privacy applications. \citet{jung2023lafd} show the usefulness of FA in a federated and differential private learning setting as well. Overall, this is similar to applications of approaches like signGD which also reduce overhead and are useful in distributed training \cite{bernstein2018signsgd, wang2021cooperative}. We leave the exploration of these connections to future work.

\textbf{Weight-Transport Problem:} The weight-transport problem is the main criticism of backpropagation in neuroscience, which has an unrealistic requirement neurons receive downstream synaptic weights to construct a backward pass \cite{grossberg1987competitive, crick1989recent}. Many approaches have tried to circumvent this problem by introducing a distinct set of feedback weights to propagate errors in the backward pass. \citet{kunin2020two} shows how these approaches can be organized into two categories: those that use layer-wise loss functions regularize the information between neighboring network layers \cite{bengio2014auto, lee2015difference, bartunov2018assessing} and those that encourage alignment between the forward and backward weights \cite{lillicrap2016random, liao2016important, moskovitz2018feedback, launay2020direct}.

The former category is broad and focuses on defining a local loss for each layer to avoid the weight-transport problem. One example is the target propagation method which encourages backward weights to locally invert forward outputs \cite{bengio2014auto, lee2015difference, meulemans2020theoretical, ernoult2022towards}. These methods work well on smaller datasets, but have struggled on larger more complicated tasks \cite{bartunov2018assessing,ernoult2022towards}.

Feedback alignment approaches address weight transport by showing backward weights don't have to equal the corresponding forward weights for effective learning \cite{lillicrap2016random, liao2016important, moskovitz2018feedback, launay2020direct}. \citet{lillicrap2016random} introduce the first feedback alignment algorithm by fixing feedback weights to random values at initialization. Surprisingly, experiments show that forward weights tend to align with their feedback weights during training \cite{liao2016important, song2021convergence}. Follow-up work improved upon this algorithm by allowing feedback weights to match the sign of the forward weights \cite{liao2016important, moskovitz2018feedback} which we refer to as sign feedback alignment (sign-FA). Other work has focused on eliminating the backward pass entirely using skip-connections \cite{nokland2016direct, launay2020direct, refinetti2021align} using a method called direct feedback alignment (DFA). 

\textbf{Feedback Alignment Theory:} Removing the need for a biological mechanism to continuously track backward weights is an important advance towards biological plausibility. Several works have focused on trying to improve our theoretical understanding of why FA methods work \cite{lillicrap2016random, lechner2020learning, refinetti2021align,song2021convergence, boopathy2022train}. In particular, some works have explored feedback alignment in linear and wide-width settings \cite{song2021convergence, boopathy2022train}. There have also been works suggesting that feedback alignment feature special dynamics that encourage learning \cite{lillicrap2016random, refinetti2021align}. There is also some work analyzing a variant of feedback alignment applied to finite and non-linear nets that they show aligns with the sign of the gradient \cite{lechner2020learning}. In comparison, our work establish similar results for (sign) FA algorithms plus convergence guarantees. 

\textbf{Deep Learning Theory:} Our conservation law is deeply connected to the implicit regularization phenomena observed in deep learning \cite{du2018algorithmic, phuong2020inductive, ji2020directional, lyu2021gradient}. \citet{du2018algorithmic} observes that the layer norms of deep ReLU networks are automatically balanced i.e increase at the same rates. \citet{ji2020directional} and \citet{lyu2021gradient} establish a result showing that deep homogeneous neural networks have weights that converge in the direction of KKT point of a margin maximization problem. Follow-up work shows that this implies generalization in certain linearly separable settings \cite{phuong2020inductive, frei2022implicit, boursier2022gradient}. This complements prior work observing benign-overfitting in deep learning models using data assumptions \cite{phuong2020inductive, frei2021proxy}. \citet{phuong2020inductive} assumes data satisfies a property called orthogonal separability and \citet{frei2021proxy} assumes data is high-dimensional. More recently, the concept of gradient dominance, or the Polyak-Łojasiewicz (PL) inequality, offers a framework for analyzing neural networks trained by gradient descent. \citet{frei2021proxy} argue that these concepts can be applied to understand the performance landscape of deep learning models.

We extend this literature in a few ways. First, we adapt and extend the work of \citet{du2018algorithmic} by deriving the implicit regularization of (sign) FA methods and then applying them to explain phenomena such as alignment and Dale's principle relevant to the bio-plausible learning community. Second, we make use of the simplifying data assumptions that have been useful for theoretical study \cite{phuong2020inductive, frei2021proxy}. Lastly, we develop an analogous criteria to gradient dominance and show how this concretely relates the alignment condition to convergence of the training objective. 

\section{Preliminaries}
\label{prelim}

In this section, we overview the basic notation, setting, background, and definitions used. Discussion of our framework and results are presented in the next section.

\textbf{Notation:} We denote scalar and vector quantities by lowercase script (e.g. $x,y,z$) and matrices by uppercase letters (e.g. $A, B, C$). When a matrix $W$ represents weights in a neural network layer we have $W_{i+1}[j,:]$ to indicate outgoing connections from the $j^{\text{th}}$ neuron in the $i^{\text{th}}$ layer. Similarly, $W_{i}[:,j]$ indicates incoming connections. For a vector $x$ we denote by $\Vert x \Vert$ the Euclidean norm. We use $\Vert \cdot \Vert_F$ and $\langle \cdot, \cdot \rangle$ to denote the Frobenius norm and inner-product, respectively. The expression $A \circ B$ represents the Hadamard (element-wise) product of two matrices or vectors of the same dimensions. We denote the trace of a square matrix $A$ by $\text{Tr}(A)$. We use the convention $\text{sign}(x) = 1$ for $z > 0$ and $\text{sign}(x) = -1$ otherwise. For any natural number $n \ge 1$ we define $[n] := \lbrace 1, \ldots, n \rbrace$. Given a differentiable function $f$ we use $\dot f$ for time derivatives and $f'$ otherwise.

\textbf{Neural Networks:} In this work, a neural network is a function $f : \mathbb{R}^d \to \mathbb{R}$ which consists of $L$ layers where the $i^{\text{th}}$ layer has width $m_i$ and operates on the input by applying a linear transformation followed by an activation function $\phi : \mathbb{R} \to \mathbb{R}$. So we write $f(x;\theta)$ to denote the output of the neural network given an input $x$ with weights defined as $\theta := (W_1, \ldots, W_L) \in \Omega$. We can express the layer pre-activation and output recursively as follows:
$$
h_i =  a_{i -1} W_i, \ a_i = \phi(h_i)
$$
where $h_i$ is the pre-activation and $a_i$ is the output for the $i^{\text{th}}$ layer. By convention we take $a_0 = x$. 

\textbf{Gradients and Bio-Plausible Optimization:} 

In this study we focus on training neural networks for binary classification tasks. Let $S := \lbrace (x_i, y_i) \rbrace_{i = 1}^n \subseteq \mathbb{R}^d \times \lbrace \pm 1 \rbrace$ be a binary-classification dataset. Recall that $f(\cdot;\theta) : \mathbb{R}^d \to \mathbb{R}$ is a neural network parameterized by $\theta := (W_1, \ldots, W_L)$. For a loss function $\ell : \mathbb{R} \to \mathbb{R}$ we define the empirical risk as follows:

$$
\mathcal{L}(f(x;\theta)) = \mathcal{L}(\theta) := \frac{1}{n} \sum_{i \in [n]} \ell(f(x;\theta) , y ) .
$$

We define $\mathcal{L}^* := \inf_{\theta} \mathcal{L}(\theta)$ as the optimal objective value. 

Backpropagation allows us to compute the contribution $\delta_i$ of each layer to the final error recursively:
$$
\delta_i = (\delta_{i+1}  W_{i+1}) \circ \phi'(h_i) , \ \delta_L = \nabla_{f} \mathcal{L}(f).
$$
We refer to the stacked vectors of contributions $\delta := (\delta_i, \ldots, \delta_L)^T$ as the backward pass. The final gradient is then given by:
$$
\nabla_{W_i} \mathcal{L}(\theta) = \delta_i a_{i-1}^T.
$$

Feedback alignment approximates the backward pass by replacing the exact transpose with feedback matrices. For each layer $i \in [L]$ we add a feedback matrix $B_i \in \mathbb{R}^{m_i \times m_{i+1}}$ which is randomly sampled at initialization when training with FA. We then modify the calculation of the backward pass as follows:
$$
\tilde \delta_i = (\delta_{i+1} \cdot  B_{i+1}) \circ \phi'(h_i) , \ \tilde \delta_L = \nabla_{f} \mathcal{L}(f)
$$

$$
\tilde \nabla_{W_i} \mathcal{L}(f_W) = \tilde \delta_i a_{i-1}^T
$$

where we have decorated quantities that differ from standard backpropagation. For sign-FA we set the feedback matrix equal to $\text{sign}(W_i)$ when training with sign-FA. When we set the forward and feedback matrices equal at initialization we will refer to this setup as adaFA following \cite{boopathy2022train}.

Throughout this work we will assume the network parameters are updated according to a flow -- infinitesimally small increments. We focus on (leaky) ReLU networks which are piece-wise linear as a function of their input and differentiable almost everywhere. For the dynamics, the parameter trajectory for the feedback alignment flow is assumed to be a continuous curve $\lbrace \theta^{(t)} | t \ge 0 \rbrace$ satisfying the following differential equation for almost all $t \ge 0$:

$$
\dot \theta^{(t)} := - \tilde \nabla \mathcal{L}(\theta^{(t)}) .
$$

The dynamics will not be fully specified by this equation due to non-differentiability on the boundary of linear regions, but we can still bound the rate of decrease. 

\begin{figure}
    \centering
    \includegraphics[width=0.4\textwidth]{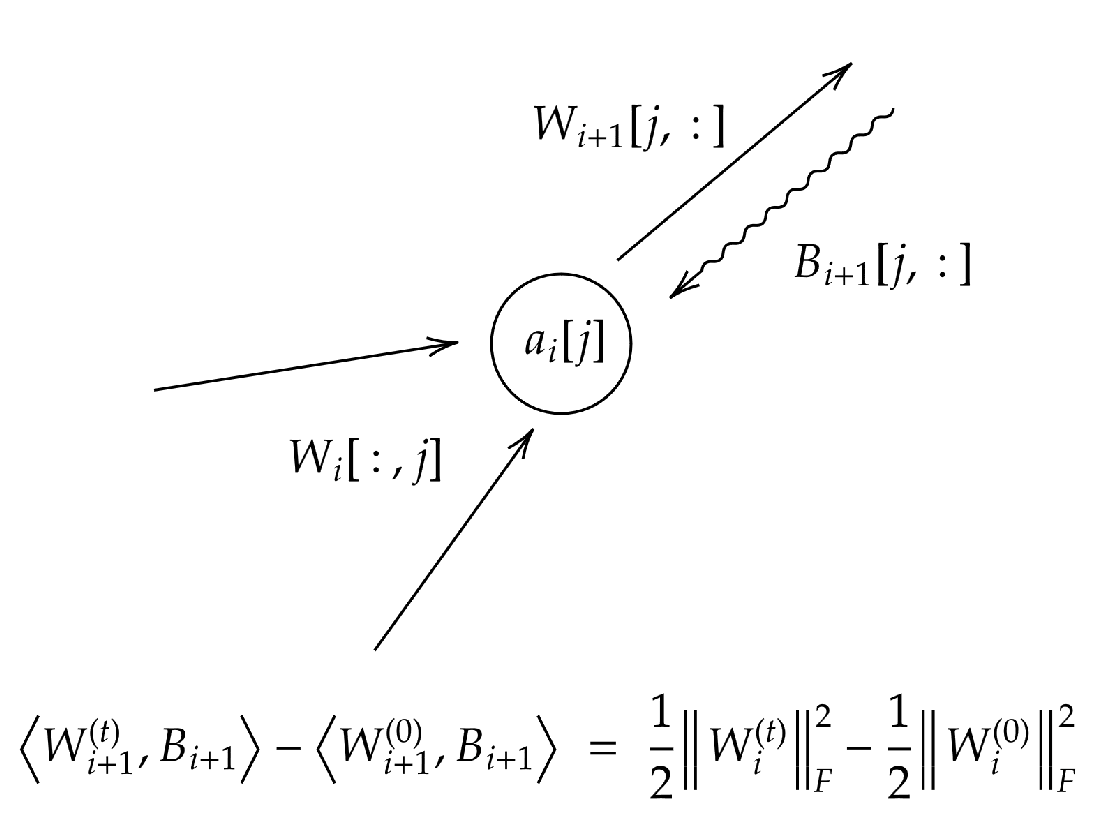}
    {\caption{\textbf{Implicit Regularization of Alignment.} We show an idealization of a neuron. The solid arrows indicate synaptic connections. In matrix form, $W_i[k,j]$ indicates a connection between the $k^{\text{th}}$ neuron in the $i^{\text{th}}$ layer to the $j^{\text{th}}$ neuron in the next layer. The matrix $B_{i+1}$ indicates a feedback connection. We establish a hard equality constraint that relates the alignment between incoming, outgoing, and feedback weights. See Theorem \ref{conservation_law} for more details.}}
    \label{fig:model}
\end{figure}

\section{Proposed Framework}
\label{framework}

In this section our goal is to develop a framework that highlights the underlying mechanisms behind feedback alignment. We propose a natural yet powerful idea -- factoring the decrease in the training objective into alignment and normative terms that model the relative "size" of the gradient and FA updating rules. This approach is supported empirically, it has been observed that feed-forward weights tend to align themselves with their respective feedback weights over the course of training \cite{lillicrap2016random}. 

We consider a binary classification dataset trained on with a neural neural network. The training dynamics under FA are the following:

$$
\dot \theta = - \tilde \nabla \mathcal{L}(\theta) , \quad \partial (\mathcal{L} \circ \theta)'(t) = \langle \nabla \mathcal{L}(\theta) , \dot \theta \rangle 
$$

$$
\Rightarrow \partial_t \mathcal{L}(\theta) = - \langle \nabla \mathcal{L}(\theta), \tilde \nabla \mathcal{L}(\theta) \rangle
$$

Our proposed framework is to factor the dynamics of the loss function into three key components that drive learning:

\begin{equation}
    \partial_t \mathcal{L}(\theta) = - \underbrace{\cos(\omega) }_{\text{angle}} \cdot \underbrace{\Vert \nabla \mathcal{L} \Vert_F}_{\text{gradient}} \cdot \underbrace{\Vert \tilde \nabla \mathcal{L} \Vert_F}_{\text{FA}}
\end{equation}

where $\omega$ is the angle between the gradient and FA update rules. This decomposition offers modularity which allows us to work towards a better understanding of bio-plausible learning rules by analyzing simpler components. In particular, positive alignment and lower-bounds on the norm terms would imply that the loss is driven down as training progresses. We show how to establish some results using the framework in Section \ref{results}. 

There are two natural definitions that help support this framing. First, we state a notion of weight alignment seen in previous work on feedback alignment and then propose a novel concept of alignment dominance which we will use later to analyze convergence.

\begin{definition}
\label{a_align}
    We say a network's forward weights $W_i$ align with their backward weights $B_i$ during training if there exists a constant $c > 0$ and time $T_c$ such that $\cfrac{\langle W_i , B_i \rangle}{\Vert W_i \Vert_F \cdot \Vert B_i \Vert_F} \ge c$ for all $t > T_c$. 
\end{definition}

This is a natural definition. Similar definitions appear in \cite{song2021convergence, lillicrap2016random}. 

\begin{definition}
\label{a_dom}
    We say a learning rule $\tilde \nabla \mathcal{L}(\theta(t))$ satisfies $(\alpha,\beta)$-alignment dominance with parameter $\alpha,\beta > 0$ if for all $\theta \in \lbrace \theta^{(t)} \rbrace_{t \ge 0}$ in the  trajectory:
$$
\langle \nabla \mathcal{L}(\theta) , \tilde \nabla \mathcal{L}(\theta) \rangle \ge \alpha \cdot (\mathcal{L}(\theta) - \mathcal{L}^*)^{\beta} .
$$

\end{definition}

Our condition is a natural extension of gradient dominance or the PL-inequality which has recently been proposed a unifying theme for studying neural network optimization \cite{frei2021proxy}. Intuitively, if an update rule aligns well with the gradient then it will enjoy similar convergence properties. The criterion we introduce  is stating that the reduction in loss using the feedback alignment, up to some constant, dominates sub-optimality gap of the network which ensures we only stop updating near a global optimum of the problem. 

In summary, we presented our framework for analyzing FA which offers insight into the auto-aligning property of FA. Additionally, the modular perspective provides a foundation for future work to integrate insights across components.

\section{Theoretical Results}
\label{results}

In this section, we present our main theoretical contributions around the Feedback Alignment (FA) method. We focus on three central themes: the relationship between alignment dominance and convergence, conservation of alignment, and auto-alignment by proper initialization. In particular, conservation of alignment is our main technical contribution and a key tool for establishing alignment dominance. Specifically, we apply this result to analyze the convergence properties of FA methods and to illustrate an application to show two-layer networks trained with (sign) FA converge. 

\subsection{Feedback Alignment During Training}

If alignment dominance for a learning rule can hold throughout training it is intuitive to think the method converges. The following result confirms this.

\begin{restatable}[Convergence under Alignment Dominance]{theorem}{convergence}
\label{convergence_result}
For any neural network with a.e. differentiable activation, trained on a fixed dataset, and using training loss $\mathcal{L}$ with $\mathcal{L}^* = 0$ such that training satisfies $(\alpha,\beta)$-alignmant dominance we have the following guarantee on the training objective:
$$
\mathcal{L}(\theta^{(t)}) \le \mathcal{L}(\theta^{(0)}) \cdot e^{-\alpha \cdot t}, \quad \beta = 1 $$
$$
\mathcal{L}(\theta^{(t)}) \le \left( \cfrac{\beta-1}{\alpha \cdot t + \frac{\beta-1}{\mathcal{L}(\theta^{(0)})}} \right)^{\cfrac{1}{\beta-1}} , \quad \beta > 1 .
$$

\end{restatable}

We defer the proof to the appendix. While our further results specifically target FA it is worth emphasizing that this condition could potentially be applied to other bio-plausible learning rules. It is known that in the kernel regime FA stays close to the gradient throughout training \cite{boopathy2022train}. It is also known that in the kernel regime we have gradient dominance \cite{frei2021proxy}. Therefore, alignment dominance will also be satisfied in the kernel-regime. However, we do not develop any finite-width approximations here since our purpose is to introduce the basic tools that are useful for working in the proposed framework. However, we consider the (shallow) finite-width case in Proposition \ref{application_two_layer} to demonstrate how these tools can be used together. 

To establish alignment dominance, it is useful to characterize a general regularization behavior of individual neurons in the network. We formalize this in the following theorem:

\begin{restatable}[]{theorem}{conservation}
\label{conservation_law}
Suppose that we apply (sign) feedback alignment to a vector-output (leaky) ReLU network trained with any differentiable loss. Then the flow of the layer weights under feedback alignment for a.e. $t \in \mathbb{R}_{\ge 0}$ maintains,
\begin{equation}
\begin{split}
\langle W_{i+1}^{(t)}[j,:], B_{i+1}^{(t)}[j,:] \rangle - \langle W_{i+1}^{(0)}[j,:], B_{i+1}^{(0)}[j,:] \rangle \\
= \frac{1}{2} \Vert W_{i}^{(t)}[:, j] \Vert_F^2 -  \frac{1}{2} \Vert W_{i}^{(0)}[:, j] \Vert_F^2 .
\end{split}
\label{con_eq}
\end{equation}
\end{restatable}

The proof of Theorem \ref{conservation_law} is available in Appendix \ref{proofs}. Figure \ref{fig:model} is useful to ground the result. It implies that the angle between forward and backward weights in the output layer is never obtuse with suitable initialization, highlighting an inherent regularization of FA towards maintaining alignment. It is also worth emphasizing that FA and sign-FA share the same conservation law.

Since the result is proven at the neuron-level there are many possible variations. For example, summing over the neurons yields a layer-wise alignment condition. It is also straight-forward to extend the result to handle weight-sharing, such as in convolution layers. In general, the theorem can be interpreted as hard-constraint on the dynamics of individual neuron weights implying that the weight trajectory lies on a lower-dimensional manifold than otherwise expected.  

For comparison to related work, \citet{lillicrap2016random} constructs a matrix-valued invariant for feedback alignment, but this only applies to linear settings. Also, \citet{du2018algorithmic} establishes a related result that applies to gradient flow. Namely, layers are "balanced" during training meaning the square of the norms change at the same rates.

The main practical application of this result is that we can give a guarantee alignment of FA during training of a scalar output network with proper initialization. It will be instructive to consider the situation with sign-FA before stating our result. 

\begin{restatable}[]{proposition}{signangle}
\label{sign_angle}
Suppose that we train a vector-output neural network with sign-FA then for any $n$ parameter layer $i \in [L]$ and a.e. $t \in \mathbb{R}_{\ge 0}$ we have:
\[
\cfrac{\langle W_i^{(t)} , B_i^{(t)} \rangle}{\Vert W_i^{(t)} \Vert_F \cdot \Vert B_i^{(t)} \Vert} \ge \cfrac{1}{\sqrt{n}} .
\]
\end{restatable}

\begin{proof}

Recall that in sign-FA the feedback weights evolve according to $B_i^{(t)} := \text{sign}(W_i^{(t)})$. It is possible to get a direct bound in this setting using the equivalence of norms which states:
$$
\Vert W_i \Vert_F \le \Vert W_i \Vert_1 = \langle W_i , B_i \rangle \le \sqrt{n} \Vert W_i \Vert_F .
$$

We can manipulate this expression directly to obtain the desired result:
$$
\Rightarrow \cfrac{\Vert W_i \Vert_F}{\Vert W_i \Vert_F \sqrt{n}} \le \cfrac{\langle W_i , B_i \rangle}{\Vert W_i \Vert_F \sqrt{n}} \le \cfrac{\sqrt{n} \Vert W_i \Vert_F}{\Vert W_i \Vert_F \sqrt{n}}$$
$$
\Rightarrow \cfrac{1}{\sqrt{n}} \le \cfrac{\langle W_i , B_i \rangle}{\Vert W_i \Vert_F \sqrt{n}} \le 1 .
$$

\end{proof}

We emphasize the argument holds for multi-output networks when considering sign-FA. We have a similar result for the output feedback weights for a scalar output network trained with FA. 

\begin{restatable}[]{lemma}{angle}
\label{angle}
Suppose that we initialize a scalar output network with $w_{L}^{(0)}[j] = b_{L}[j] = 1$ such that $\Vert w_{L-1}^{(0)}[j] \Vert < \sqrt{2} \cdot \Vert w_{L}^{(0)}[j] \Vert$ then for a.e. $t \in \mathbb{R}_{\ge 0}$ we have $w_L^{(t)} > 0$ and the following bound:
\[
\cfrac{\langle w_L^{(t)} , b_L^{(t)} \rangle}{\Vert w_L^{(t)} \Vert_F \cdot \Vert b_L^{(t)} \Vert} \ge \cfrac{1}{\sqrt{n}} .
\]
\end{restatable}

\begin{proof}

By assumption, we initialize $w_{L}^{(0)}[j] = b_{L}[j]$ such that $\Vert w_{L-1}^{(0)}[j] \Vert < \sqrt{2} \cdot \Vert w_{L}^{(0)}[j] \Vert$ and so Theorem \ref{conservation_law} implies the following:
$$
\langle w_{L}^{(t)}[j], b_{L}[j] \rangle $$
$$= \frac{1}{2} \Vert w_{L-1}^{(t)}[j] \Vert_F^2 + (\langle w_{L}^{(0)}[j], b_{L} \rangle -  \frac{1}{2} \Vert w_{L-1}^{(0)}[j] \Vert_F^2)$$
$$
= \frac{1}{2} \Vert w_{L-1}^{(t)}[j] \Vert_F^2 + (\Vert w_{L}^{(0)}[j] \Vert_F^2 -  \frac{1}{2} \Vert w_{L-1}^{(0)}[j] \Vert_F^2)$$
$$
\ge \frac{1}{2} \Vert w_{i}^{(t)}[j] \Vert_F^2  > 0.
$$
In the last step we applied our norm assumption. This means forward weights match the sign of their feedback weights in the output layer so we are done after applying equivalence of norms as in Lemma \ref{sign_angle}. 

\end{proof}

\begin{figure}[ht]
\centering
\includegraphics[width=1.0\linewidth]{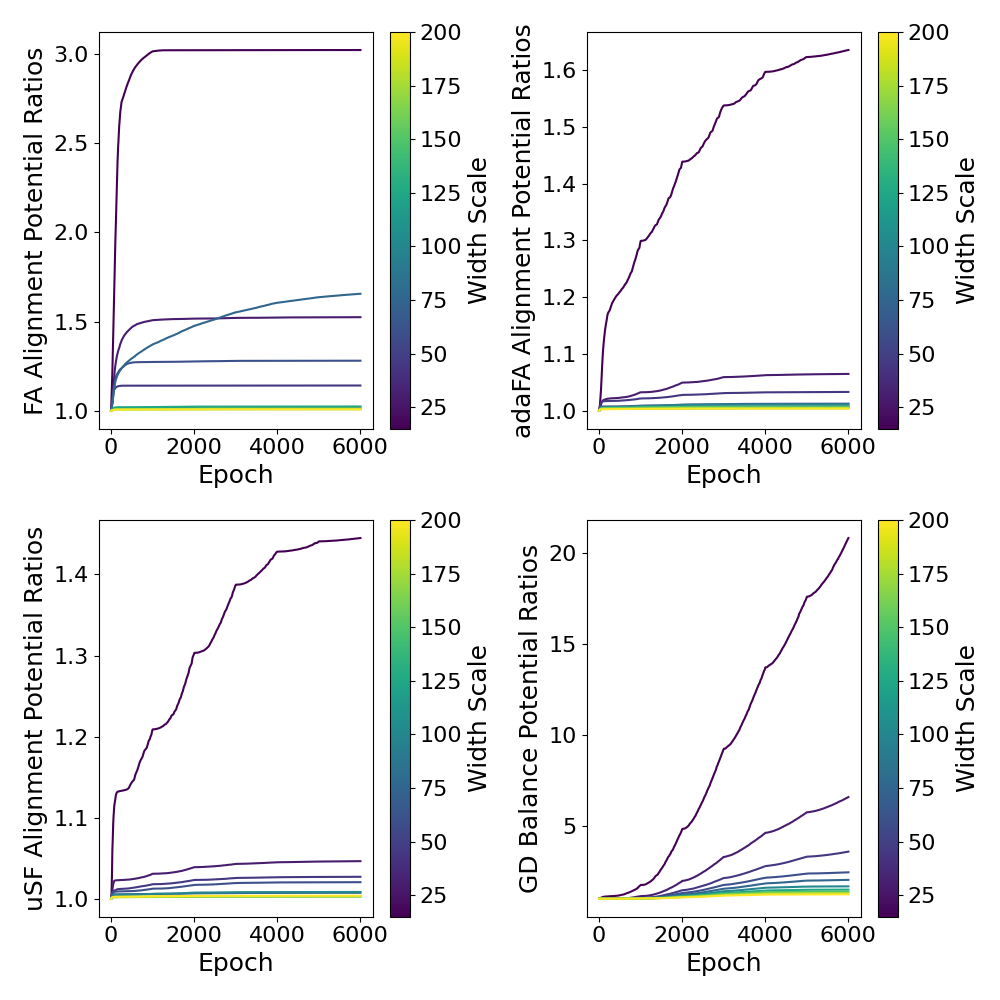}
\caption{\textbf{Conservation of Alignment.} We measure the "potential" $\langle W_{2}^{(t)}, B_{2} \rangle -  \frac{1}{2} \Vert W_1(t) \Vert_F^2$ during training on MNIST with varying width. We normalize by the value at initialization and so Theorem \ref{conservation_law} predicts this measure to be constant and equal to one. For gradient descent there is a similar result; change of each layer's squared norm is equal \cite{du2018algorithmic}. See Section \ref{empirical} for more details.}
\label{fig:Conservation}
\end{figure}

The corollary shows that certain initializations can inherently regularize the network towards maintaining alignment between forward and backward weights during training. This could be compared with \cite{song2021convergence} who show a matching upper-bound on alignment in linear settings. Together, these results help better characterize alignment as a function of network width. We also want to emphasize that this result implies that two-layer neural networks trained with FA can have output neurons with constant sign. This mirrors Dale's principle in neuroscience states that neurons tend to have synaptic connections that are either positive or negative \cite{eccles1976electrical, cornford2021learning}. \citet{lechner2020learning} studies a variant of FA that is monotone which means it satisfies this principle. Here we show how the principle can emerge naturally. 

\subsection{Connecting Alignment to Convergence}

In the previous sections, we presented a framework and foundational tools that help explain alignment and convergence in Feedback Alignment. This section aims to focus on the direct influence of alignment on convergence guarantees, offering a more granular understanding of how our framework can be used to analyze FA. We consider the following types of datasets, as defined below. 

\begin{restatable}{definition}{basicdata}
\label{s_ortho}
    We say that a dataset $S := \lbrace (x_i, y_i) \rbrace_{i \in [n]}$ is $\gamma$-orthogonal separable whenever $\langle x_i y_i , x_j y_j \rangle \ge \gamma$ for some constant $\gamma > 0$.
\end{restatable}

\begin{restatable}{definition}{linedata}
\label{n_ortho}
    We say that a dataset $S := \lbrace (x_i, y_i) \rbrace_{i \in [n]}$ is $(\gamma, \epsilon)$-nearly orthogonal whenever $\min_i \Vert x_i \Vert_2^2 \ge n \cdot (\gamma + \epsilon)$ where $\gamma := \max_{i \not = j} \vert \langle x_i,  x_j \rangle \vert > 0$ and $\epsilon > 0$.
\end{restatable}

\begin{figure*}[ht]
\centering
\includegraphics[width=1.0\textwidth]{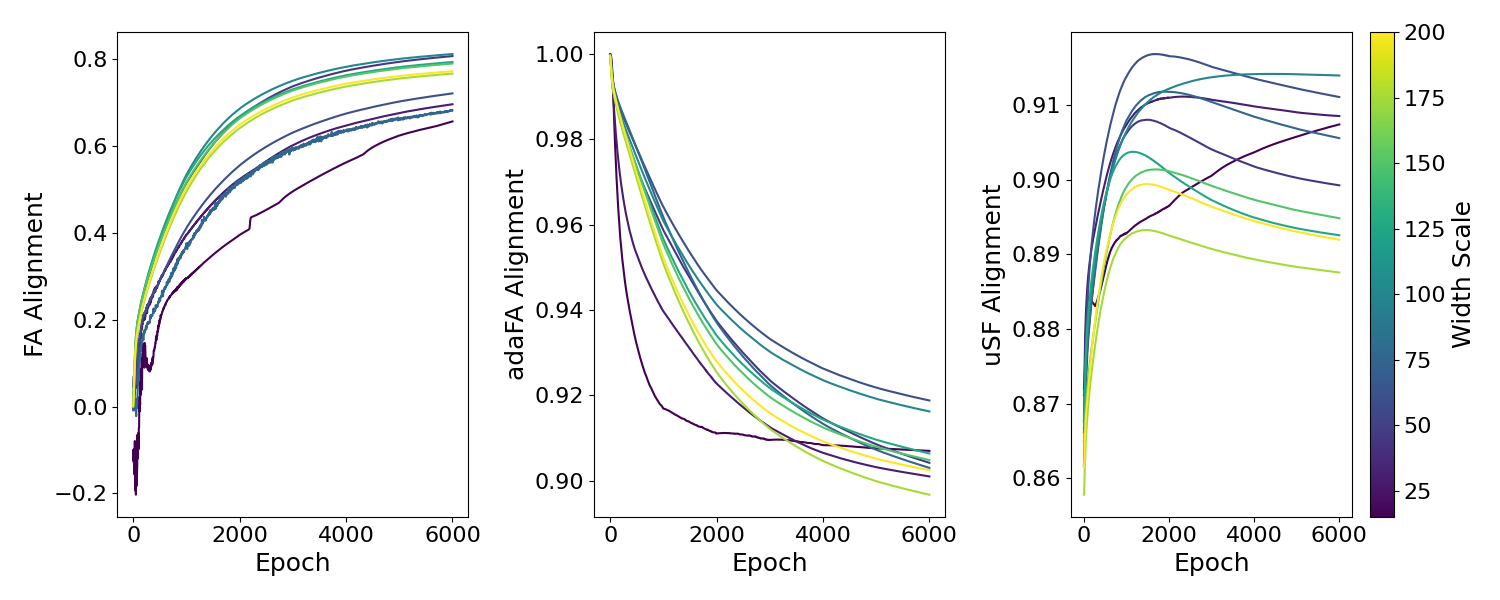}
\caption{\textbf{Alignment During Training:} We plot the cosine of the angle between forward ($W_2^{(t)}$) and feedback ($B_2$) weights during training on Noisy MNIST with 20\% label noise across varying network widths. See Section \ref{empirical} for more details.}
\label{fig:Alignment_Evolution}
\end{figure*}

These types of datasets have previously appeared recently in the deep learning literature to investigate implicit regularization towards benign overfitting \cite{phuong2020inductive,frei2022implicit}. In particular, any dataset $S$ satisfying either orthogonality property is linearly separable with large margin. We leverage these assumptions to establish our convergence guarantees for (sign) FA:

\begin{restatable}{proposition}{application}
\label{application_two_layer}

For any two-layer network parameterized by $\theta := (W_1, w_2)$ with leaky ReLU activation $\phi(z) = \max(c' \cdot z, z)$ where $c' > 0$, exponential loss $\mathcal{L}$, and dataset $S$ that is $\gamma$-orthogonal separable or $\gamma$-nearly orthogonal then we have the following:

\begin{enumerate}
    \item If the initialization satisfies $\Vert W_1[:, j]^{(0)} \Vert \le \sqrt{2} \cdot \Vert w_{2}[j]^{(0)} \Vert$ then training with sign feedback alignment satisfies alignment dominance.
    \item If additionally, we have $w_{2}[j] = b_{2}[j] = 1$ then feedback alignment satisfies alignment dominance.
\end{enumerate}

\end{restatable}

This proposition combined with Theorem \ref{convergence} implies that two-layer networks indeed converge. 

\begin{hproof}
We consider a two-layer network with parameters $\theta = (W_1, w_2)$, and the learning objective defined by the exponential margin-loss. The goal is to establish $(\gamma,2)$-alignment dominance. This involves showing a strong correlation between the gradient and the feedback alignment for the first layer. To analyze correlation we derive expressions for the updates under the feedback alignment flow. This factors into an alignment component $\langle w_2, T b_2 \rangle$ modulated by a diagonal matrix and a data correlation component $\sum_{ij} \langle y_i x_i, y_j x_j \rangle$. We use Lemma \ref{angle} to analyze the alignment component. For the data correlation we show that for orthogonally separable or nearly orthogonal datasets, this component satisfies a lower bound.
\end{hproof}

\textbf{Comparisons and Limitations:} Compared with previous work on feedback alignment we are the first to study (sign) FA in \textit{non-linear} settings distinct from previous work focused on settings in the kernel regime \cite{song2021convergence, boopathy2022train}. Additionally, our result on alignment conservation stated in Theorem \ref{conservation} is general, it holds for all (leaku) ReLU networks trained with (sign) FA. 

Our analysis, while informative, does have limitations. While Theorem \ref{convergence_result} implies convergence generally, establishing it in particular settings requires careful analysis. In particular, Proposition \ref{application_two_layer} requires that the activation be (leaky) ReLU and the training data be linearly separable. On the other hand, extending convergence results to deep networks requires sophisticated analysis which is an active area of study in deep learning \cite{ji2020directional, lyu2021gradient}. In general, we expect $(\alpha,\beta)$-alignment is satisfied in many settings, such as the kernel regime \cite{jacot2018neural}, which we do not explore in this work. Exploring these settings is an important direction for future study. Finally, our theory only bounds alignment and does not directly explain why the alignment can increase over time. We think this is a challenging direction. Our experiments in the next section show alignment can exhibit non-monotonic behavior and switch from non-aligned to aligned.

\section{Empirical Analysis of Alignment and Multi-Class Performance}
\label{empirical}

In this section we empirically evaluate feedback alignment (FA) mechanisms in neural networks and their performance in multi-class classification tasks. We test two hypotheses:

\begin{enumerate}
    \item The conservation of learning dynamics as per Theorem \ref{conservation_law}, Lemma \ref{angle}, and Lemma \ref{sign_angle} hold under practical training conditions. 
    \item Alignment with the true gradient enhances multi-class classification performance.
\end{enumerate}

Our theory makes direct contact with experiments, confirming key quantitative predictions and consistency with previous results. We also analyze datasets exhibiting benign overfitting, where our theory suggests FA methods may have similar implicit regularization. Our experiments provide initial evidence for this unexplained phenomena in FA. These results explain alignment and discover potential benign overfitting in FA, inspiring future work on the underlying mechanisms.

We assess the first hypothesis by measuring the deviation from the conservation law (\ref{con_eq}). The second hypothesis is evaluated using classification accuracy on validation sets. Our experiments involve widely-used datasets -- MNIST, CIFAR-100, TinyImageNet \cite{lecun1998mnist, krizhevsky2009learning, le2015tiny}. We use a pre-existing python package for FA implementations \cite{sanfiz2021benchmarking}. We compare full-batch gradient descent, signFA, FA, and adaFA. Training is conducted on two A100 GPUs. Refer to Appendix \ref{more_details} for detailed experimental settings.

\subsection{Two-Layer Fully-Connected Network}
\label{mnist}

We investigate two-layer leaky ReLU networks with widths ranging from 15 to 200, trained on a noisy MNIST subset with 20\% training label noise. The test set is clean. Our protocol involves a 6,000-epoch training schedule with adaptive learning rates. Experimental iterations are repeated multiple times so results are reported with standard error bars. Further experimental details are in Appendix \ref{more_details}.

\textbf{Conservation:} The experimental results, as displayed in Figure \ref{fig:Conservation} are consistent with the predicted conservation law (Theorem \ref{conservation_law}), with larger deviations occurring in smaller width networks and for FA. This suggests initialization's has an impact on learning dynamics. See Appendix \ref{more_results} for a numerical analysis of the deviations. 

\textbf{Alignment Dynamics:} Figure \ref{fig:Alignment_Evolution} shows consistent improvement in alignment over epochs, with adaFA maintaining higher alignment than FA, indicating the significant role of initialization. Sign-FA exhibits complex, non-monotonic alignment behavior, yet it consistently exceeds our theoretical lower bounds, affirming Lemma \ref{sign_angle}. 

\textbf{Generalization:} Training and test losses are reported in Figure \ref{fig:mnist_double_descent}. Across methods and network widths, higher parameterization and overfitting does not necessarily impair generalization, highlighting FA's implicit regularization properties.

\begin{figure}[ht]
\centering
\includegraphics[width=1.0\linewidth]{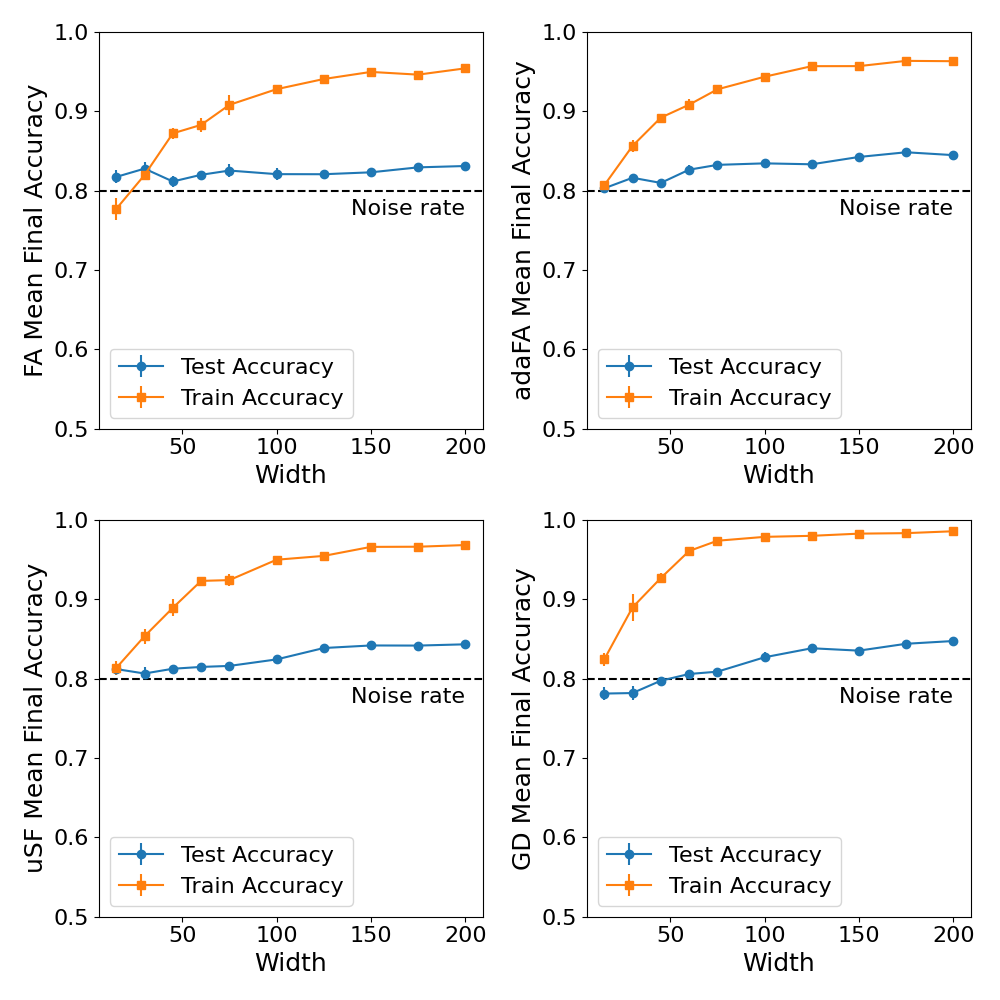}
\caption{\textbf{Benign Overfitting:} We plot the test-loss curve as a function of width for networks trained on noisy MNIST with 20\% label noise. Despite fitting the noise all methods are able to generalize.}
\label{fig:mnist_double_descent}
\end{figure}

\subsection{Deep Convolutional Networks}

We extend our analysis to deeper networks -- LeNet on CIFAR-100 and ResNet-18 on Tiny-ImageNet \cite{lecun1998gradient, he2016deep}. We randomly sub-sample $n$-class subsets from these datasets for $n \in [10,25,50]$ to explore the effect of classification difficulty on performance. For each dataset, we tune hyper-parameters on a random $10$-class and report results on a fresh sample. The test set for Tiny-ImageNet is unlabeled so we use the validation set as the test set. For more discussion of implementation details, see Appendix \ref{more_details}.

\textbf{Effect of Alignment Strategies:} We report mean test accuracy with standard errors in Tables \ref{initialization_effect} and \ref{image_effect}. The choice of FA method plays an important role in network performance. Initialization impacts performance, particularly in simpler classification tasks. This is evident from the results on CIFAR-100. However, as class complexity increases, the relative advantage of sophisticated initialization strategies appears to diminish. While performance generally declines with an increasing number of classes, sign-FA demonstrates more robust performance than adaFA, maintaining competitiveness with traditional backpropagation. This suggests alignment is more important than initialization for effective multi-class performance. Further analysis, of layer-wise alignment for the LeNet trained on CIFAR-100 is available in Appendix \ref{deep_alignment}.

\begin{table}[t]
\caption{Test-Accuracy on $n$-class Subsets of CIFAR-100}
\label{initialization_effect}
\vskip 0.15in
\begin{center}
\begin{scriptsize}
\begin{sc}
\begin{tabular}{|l|c|c|c|}
\hline
Method & 10 Classes & 25 Classes &  50 Classes \\ \hline
FA & $55.2\% \pm 3.8\%$ & $39.7 \% \pm 0.5\%$ & $32.6 \% \pm 0.6\%$ \\
adaFA & $61.6\% \pm 0.4\%$ & $42.9\% \pm 1.2\%$ & $34.1 \% \pm 0.5\%$ \\
signFA & $74.1\% \pm 0.1\%$ & $58.9\% \pm 1.2\%$  & $48.9 \% \pm 0.3\%$\\
GD & $75.8\% \pm 0.2\%$ & $60.1\% \pm 1.2\%$ & $50.6 \% \pm 0.4\%$ \\
\hline
\end{tabular}
\end{sc}
\end{scriptsize}
\end{center}
\vskip -0.1in
\end{table}

\begin{table}[t]
\caption{Test-Accuracy on $n$-class Subsets of Tiny-ImageNet}
\label{image_effect}
\vskip 0.15in
\begin{center}
\begin{scriptsize}
\begin{sc}
\begin{tabular}{|l|c|c|c|}
\hline
Method & 10 Classes & 25 Classes &  50 Classes \\ \hline
FA & $49.6\% \pm 1.7\%$ & $40.1 \% \pm 1.6\%$ & $30.9 \% \pm 0.8\%$ \\
adaFA & $51.6\% \pm 2.2\%$ & $35.3\% \pm 1.0\%$ & $31.0 \% \pm 0.9\%$ \\
signFA & $64.9\% \pm 4.8\%$ & $53.0\% \pm 2.6\%$  & $47.7 \% \pm 0.6\%$\\
GD & $67.5\% \pm 0.7\%$ & $55.9\% \pm 0.6\%$ & $53.4 \% \pm 2.3\%$ \\
\hline
\end{tabular}
\end{sc}
\end{scriptsize}
\end{center}
\vskip -0.1in
\end{table}

\section{Conclusion and Future Work}
\label{conclusion}

In this work, we introduced a novel theoretical framework to analyze Feedback Alignment (FA) methods, offering a deeper understanding that we applied to explain the success of aligned FA methods such as adaFA and sign-FA. Our exploration revealed a conservation law that shows an inherent regularization in FA towards maintaining alignment between forward and backward weights. Looking ahead, this research paves the way for several key areas of exploration. Assessing FA's robustness to distribution-shift could be a valuable direction for future work. Perhaps, the feedback weights can be interpreted as a prior affording better generalization than under gradient descent. Additionally, the integration of FA with diverse learning paradigms could unveil more efficient and powerful algorithms. The sign-FA method could be seen as a one-bit approximation of the weights for the backward pass which may indicate applications to differential-privacy. In summary, there is a compelling opportunity for interdisciplinary collaboration with neuroscientists to explore its relevance in understanding biological neural networks.

\section*{Impact Statement} This paper presents work with the goal to advance the field of bio-plausible learning. There are many potential societal consequences of more efficient deep learning, none that need to be specifically highlighted here.

\section*{Acknowledgement} We are grateful to Jeremy Cohen for valuable early discussions and feedback on an initial draft of this paper. This research was funded in part by the Accelerate Foundation Models Research program at Microsoft. Zachary Robertson is supported by a Stanford School of Engineering fellowship.

\bibliography{references}
\bibliographystyle{icml2024}

%%%%%%%%%%%%%%%%%%%%%%%%%%%%%%%%%%%%%%%%%%%%%%%%%%%%%%%%%%%%%%%%%%%%%%%%%%%%%%%
%%%%%%%%%%%%%%%%%%%%%%%%%%%%%%%%%%%%%%%%%%%%%%%%%%%%%%%%%%%%%%%%%%%%%%%%%%%%%%%
% APPENDIX
%%%%%%%%%%%%%%%%%%%%%%%%%%%%%%%%%%%%%%%%%%%%%%%%%%%%%%%%%%%%%%%%%%%%%%%%%%%%%%%
%%%%%%%%%%%%%%%%%%%%%%%%%%%%%%%%%%%%%%%%%%%%%%%%%%%%%%%%%%%%%%%%%%%%%%%%%%%%%%%
\newpage
\appendix
\onecolumn

\section{Implementation Details}
\label{more_details}

All experiments are run on two A100 GPUs. We run all experiments multiple times. We compare the following methods: FA, adaFA, signFA, and GD. 

\textbf{Noisy MNIST Experiments:} The dataset for each training run consists of a random 4k subset of the MNIST dataset, with 20\% label noise. We train two-layer neural network architectures with leaky-ReLU activation. Hyper-parameter settings are the same for all methods and tuned to minimize the training loss while keeping low momentum to observe benign over fitting. Networks are trained on cross-entropy loss.  These networks have their weights initialized to $[-1/\sqrt{m} , 1/\sqrt{m}]$ where $m$ is the width of the network. We use full-batch training and train for 6k epochs. We have the initial learning rate set to 0.05, no weight decay, and a momentum of 0.05. The learning rate was scheduled to decrease by a factor of ten every 1,000 epochs. We track the following variables every 5 epochs: train loss, test loss, train accuracy, test accuracy, layer norms, and the dot-product between forward and feedback weights. We varied the number of hidden units across the range $[15,30,45,60,75,100,125,150,175,200]$. This procedure is repeated 6 times for each width setting and then we report the mean and standard error of the final train and test accuracy. 

\textbf{CIFAR-100 Experiments:} We train LeNet architectures on randomly selected $n$-class subsets of CIFAR-100 where $n \in [10,25,50]$. We normalize the mean and std of the data before it is feed into the network. Hyper-parameter settings are the same for all methods and tuned from the default settings of biotorch on a random $10$-class subset. We lowered the learning rate, increased the batch-size, and compensate by training for more epochs. Networks are trained on cross-entropy loss.  Xavier initialization is used for the network. We use a batch size of 1024, and train for 500 epochs. The initial learning rate is 0.01, weight decay is 0.0001, and momentum is 0.9. The learning rate was scheduled to decrease by half at the 100th and 250th epochs. We track the following variables every epoch: train loss, test loss, train accuracy, test accuracy, layer norms, and the dot-product between forward and feedback weights in each layer. This procedure is repeated six times for each method over the range $n \in [10,25,50]$ and then we report the mean and standard error of the final train and test accuracy. 

\textbf{Tiny-ImageNet Experiments:} We train ResNet-18 architectures on randomly selected $n$-class subsets of Tiny-ImageNet where $n \in [10,25,50]$. We normalize the mean and std of the data before it is feed into the network. Hyper-parameter settings are the same for all methods and tuned from the default settings of biotorch on a random $10$-class subset. We lowered the learning rate and decreased the number of epochs. Networks are trained on cross-entropy loss.  Xavier initialization is used for the network. We use a batch size of 512, and train for 50 epochs. The initial learning rate is 0.01, weight decay is 0.0001, and momentum is 0.9. The learning rate was scheduled to decrease by half at the 20th and 40th epochs. We track the following variables every epoch: train loss, test loss, train accuracy, test accuracy. This procedure is repeated four times for each method over the range $n \in [10,25,50]$ and then we report the mean and standard error of the final train and test accuracy.

\newpage

\section{Further Experiment Details and Results}
\label{more_results}

\subsection{Testing Conservation of Alignment}

Theorem \ref{conservation_law} predicts the following quantity is equal to one throughout training:

\[
\begin{aligned}
    \cfrac{\langle W_{i+1}^{(t)}[j,:], B_{i+1}^{(t)}[j,:] \rangle - \frac{1}{2} \Vert W_{i}^{(t)}[:, j] \Vert_F^2}{\langle W_{i+1}^{(0)}[j,:], B_{i+1}^{(0)}[j,:] \rangle -  \frac{1}{2} \Vert W_{i}^{(0)}[:, j] \Vert_F^2} = 1.
\end{aligned}
\]

We provide a numerical breakdown of the mean absolute deviations from one in Figure \ref{fig:Conservation} as a table. See Section \ref{mnist} for more details. Values are reported to the first significant decimal with standard errors over the four runs for each width.

\begin{table}[t]
\caption{Absolute Mean Deviation from Conservation Law ± SE by Width and Method}
\label{deviation_by_method_and_width}
\vskip 0.15in
\begin{center}
\begin{small}
\begin{sc}
\begin{tabular}{|l|c|c|c|c|}
\hline
Width & FA & adaFA & uSF & GD \\ \hline
15 & $1.5 \pm 0.6$ & $1.0 \pm 0.2$ & $0.4 \pm 0.1$ & $22.3 \pm 3.2$ \\
30 & $0.5 \pm 0.2$ & $0.1 \pm 0.02$ & $0.07 \pm 0.02$ & $4.8 \pm 0.3$ \\
45 & $0.2 \pm 0.1$ & $0.03 \pm 0.01$ & $0.02 \pm 0.005$ & $2.6 \pm 0.1$ \\
60 & $0.2 \pm 0.1$ & $0.01 \pm 0.002$ & $0.02 \pm 0.003$ & $1.4 \pm 0.03$ \\
75 & $0.2 \pm 0.2$ & $0.01 \pm 0.002$ & $0.009 \pm 0.003$ & $1.0 \pm 0.09$ \\
100 & $0.07 \pm 0.05$ & $0.006 \pm 0.002$ & $0.01 \pm 0.002$ & $0.6 \pm 0.04$ \\
125 & $0.01 \pm 0.002$ & $0.006 \pm 0.001$ & $0.005 \pm 0.001$ & $0.4 \pm 0.04$ \\
150 & $0.008 \pm 0.002$ & $0.004 \pm 0.001$ & $0.004 \pm 0.0004$ & $0.4 \pm 0.01$ \\
175 & $0.007 \pm 0.002$ & $0.005 \pm 0.0004$ & $0.003 \pm 0.0006$ & $0.3 \pm 0.008$ \\
200 & $0.01 \pm 0.002$ & $0.004 \pm 0.0004$ & $0.003 \pm 0.0002$ & $0.2 \pm 0.004$ \\
\hline
\end{tabular}
\end{sc}
\end{small}
\end{center}
\vskip -0.1in
\end{table}

\subsection{Analysis of Alignment in LeNet}
\label{deep_alignment}

A LeNet architecture consists of fully-connected and convolutional layers interleaved with max pooling operations. We analyze the alignment between the forward weights in trainable layers and their feedback weights. The ordering of trainable layers is as follows:

$$
C1 \to C2 \to C3 \to F1 \to F2
$$

where $C1,C2,C3$ are the convolutional layers and $F1, F2$ are the fully-connected layers.

We train the LeNet on $n$-class subsets of CIFAR-100 and display the alignment of each layer with it's corresponding feedback weights during training. For more implementation details see Appendix \ref{more_details}.

\begin{figure}[H]
\centering
\includegraphics[width=0.85\textwidth]{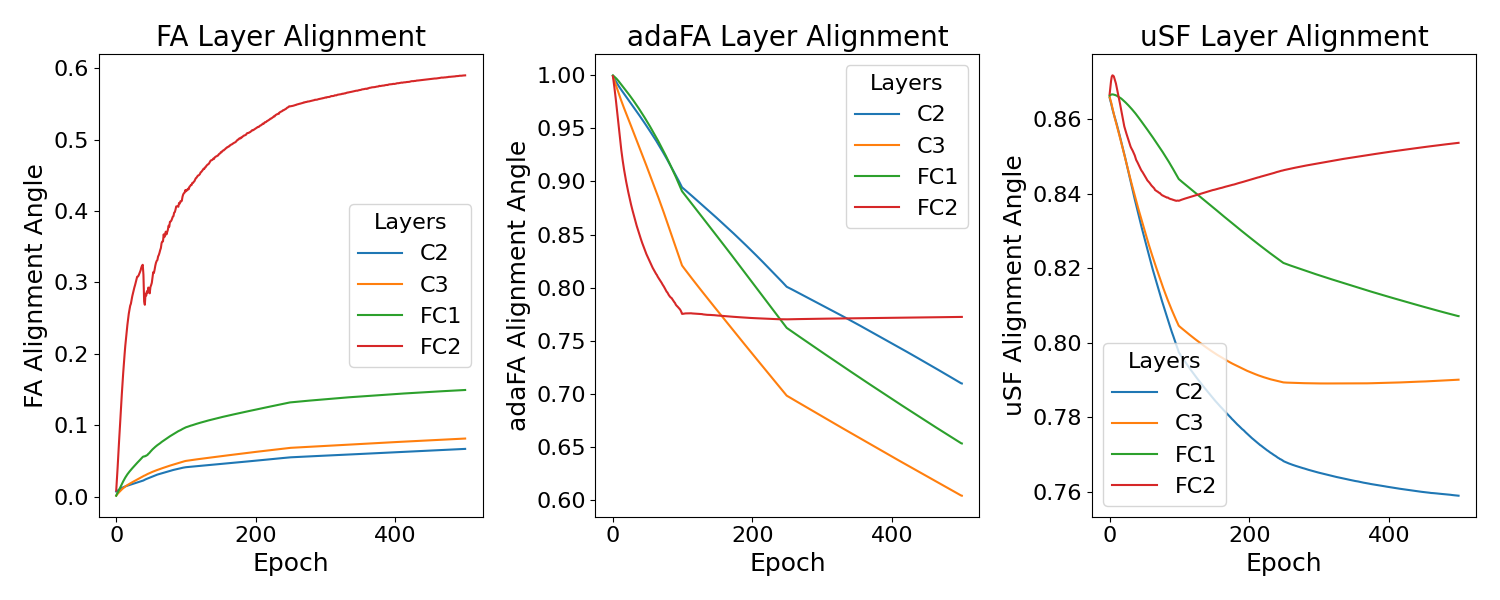}
\caption{\textbf{10-Class Alignment:} We plot the cosine of the angle between forward and backward weights during training on $10$-class subsets of CIFAR-100.}
\end{figure}

\begin{figure}[H]
\centering
\includegraphics[width=0.85\textwidth]{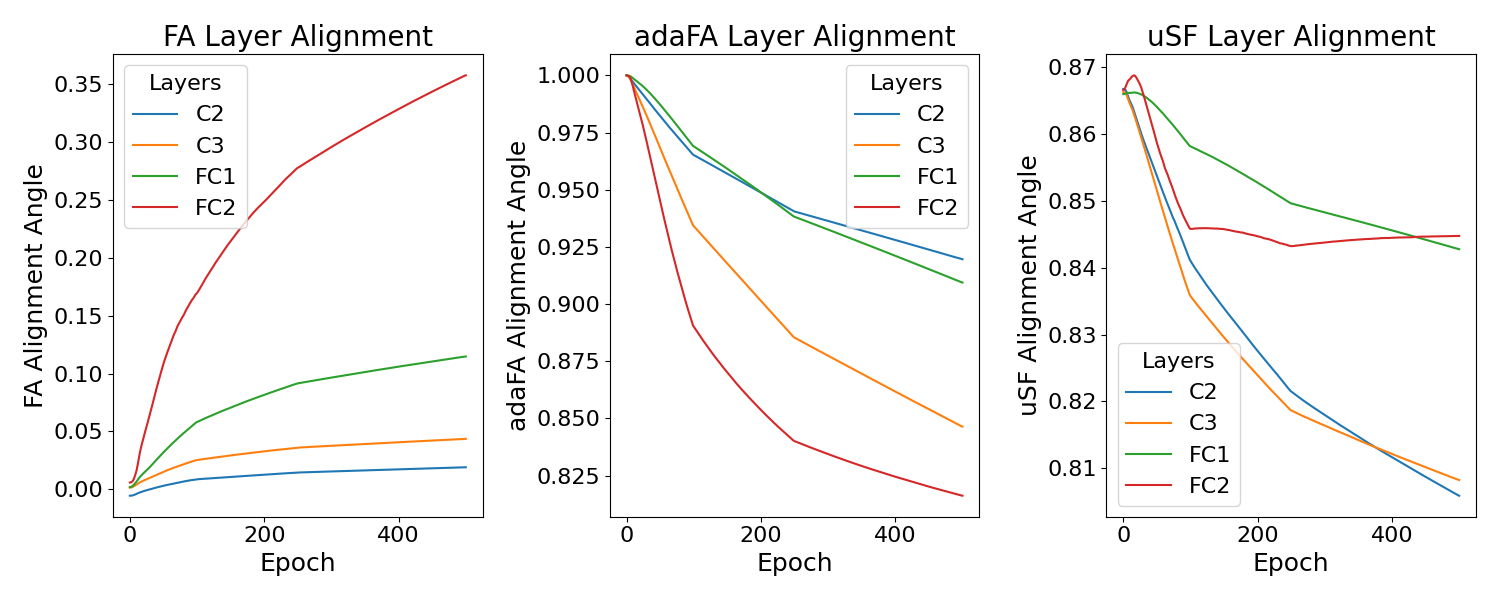}
\caption{\textbf{25-Class Alignment:} We plot the cosine of the angle between forward and backward weights during training on $25$-class subsets of CIFAR-100.}
\end{figure}

\begin{figure}[H]
\centering
\includegraphics[width=0.85\textwidth]{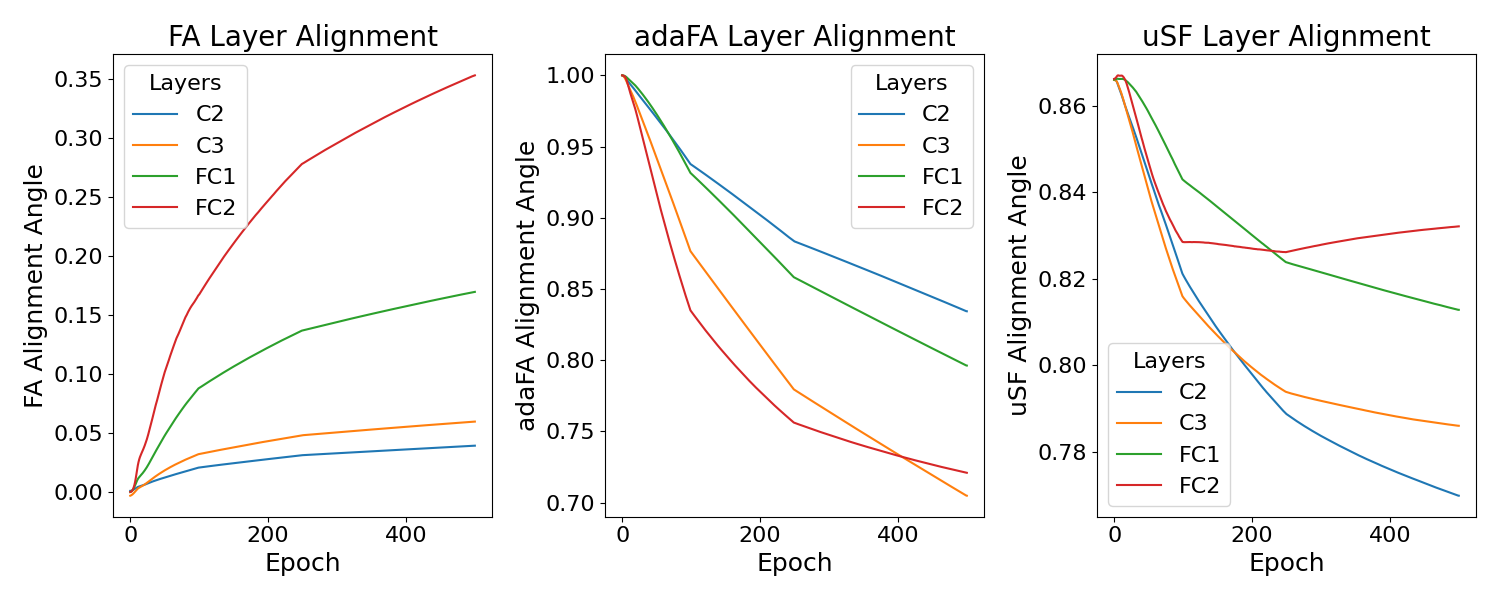}
\caption{\textbf{50-Class Alignment:} We plot the cosine of the angle between forward and backward weights during training on $50$-class subsets of CIFAR-100.}
\end{figure}

These figures show consistent improvement in alignment over epochs for FA, particularly the output layer. Moreover, adaFA maintains higher alignment than FA, indicating the significant role of initialization. Overall, these results are similar to our analysis with MNIST. We also find that alignment tends to degrade as the number of classes is increased. This effect is more pronounced for FA, but is also visible for adaFA and to a lesser extend sign-FA.

\newpage

\section{Omitted Proofs}
\label{proofs}

\subsection{Proof of theorem \protect\ref{conservation_law}}
\label{conservation}

\conservation*

\begin{proof}
For the moment assume that we want the dynamics on the interior of a differentiable region. We drop time indexing unless needed. We consider a differentiable loss such as $\ell_k : \mathbb{R}^m \times \mathbb{R}^m \to \mathbb{R}$ where we take $\delta_L(k) := \nabla_{f} \ell(f,y_k)$ as the gradient with respect to the output. In Section \ref{prelim} we review further notation.

Let $A_i$ be a diagonal matrix with activations of the $i$-th output layer on the diagonal,
$$
A_i = \text{diag} \circ \sigma' \circ W_i \circ  \sigma \circ \ldots \circ \sigma \circ W_1 (x)
$$
This matrix indicates if an output at the $i$-th layer is non-zero. Supressing composition we have,
$$
f_{W}(x) = x W_1 A_1 \ldots A_{L-1} W_L $$

First we calculate the gradient and then compare with the feedback alignment update rule. 

$$
\Rightarrow \nabla_{W_i} f_W(x) = (x W_1 A_1 \ldots W_{i-1} A_{i-1})^T (A_i W_{i+1} \ldots A_{L-1} W_L)^T
$$
Feedback aligment simplifies the backward pass by replacing terms with random feedback matrices. We have the following for the dynamics under feedback alignment:
\[
\dot W_{i} = -\eta \cdot \mathbb{E}_k \left[(x_k W_1 A_1 \ldots W_{i-1} A_{i-1})^T \delta_L(k)^T (A_i B_{i+1} \ldots A_{L-1} B_L)^T \right] .
\]
We let $W_i[:, j]$ equal $W_i$ on the $j^{\text{th}}$ row and zero otherwise. This forms weights that activate neuron $j$ in layer $i+1$. So we have,
\[
\begin{aligned}
\langle \dot W_i[:, j], W_i[:, j] \rangle &= \langle \dot W_i, W_i[j,:] \rangle = \text{Tr}(\dot W_i^T W_i[j,:]) \\
&= - \eta \cdot \mathbb{E}_k \left[ \text{Tr}((A_i B_{i+1} \ldots A_{L-1} B_L) \delta_L(k) (x_k W_1 A_1 \ldots W_{i-1} A_{i-1}) W_i[:, j]) \right] \\
&= - \eta \cdot \mathbb{E}_k \left[ \text{Tr}(x_k W_1 A_1 \ldots W_i[:, j] A_i B_{i+1} \ldots B_L \delta_L(k))\right]
\end{aligned}
\]
We are making use of the trace representation for the inner-product and the cyclic property of the trace map. Similarly, for the next layer we obtain,
\[
\begin{aligned}
\langle \dot W_{i+1}, B_{i+1}[j,:] \rangle &= \text{Tr}(\dot W_{i+1}^T B_{i+1}[j,:]) \\
&= - \eta \cdot \mathbb{E}_k \left[  \text{Tr}((A_{i+1} B_{i+2} \ldots A_{L-1} B_L) \delta_L(k) (x_k W_1 A_1 \ldots W_{i} A_{i}) B_{i+1}[j,:]) \right] \\
&= - \eta \cdot \mathbb{E}_k \left[ \text{Tr}(x_k W_1 A_1 \ldots W_i A_i B_{i+1}[j,:] \ldots B_L \delta_L(k))\right]
\end{aligned}
\]
Now we show that $W_i[:, j] A_i B_{i+1} = W_i A_i B_{i+1}[j,:]$. For the left-side notice that only neuron $j$ will be the only non-zero activated. Additionally, neuron $j$ will be activated and equal as though we didn't mask. Now consider the right side. We only consider the output from neuron $j$ which we just showed is the same as considering the left-side so we have equality. The major implication is that,
\[
\langle \dot W_i[:, j], W_i[:, j] \rangle = \langle \dot W_{i+1}[j,:], B_{i+1}[j,:] \rangle.
\]
The trace map is linear so after an integration by parts we have that,
\[
\Rightarrow \int_0^t \text{Tr}(\dot W_i^{(s)}[:, j] (W_i^{(s)}[:, j])^T) ds = \text{Tr} \left[ \int_0^t \dot W_i^{(s)} [:, j] (W_i^{(s)}[:, j])^T \right] ds
= \frac{1}{2} \Vert W_i^{(t)}[:, j] \Vert_F^2 -  \frac{1}{2} \Vert W_i^{(0)}[:, j] \Vert_F^2
\]
Finally,
\[
\begin{aligned}
\int_0^t \text{Tr}(\dot W_{i+1}^{(s)}[j,:] B_{i+1}^{(s)}[j,:]^T) ds = \text{Tr} \left[\int_0^t \dot W_{i+1}^{(s)}[j,:] B_{i+1}^{(s)}[j,:]^T ds \right] \\
 = \langle W_{i+1}^{(t)}[j,:], B_{i+1}^{(t)}[j,:] \rangle - \langle W_{i+1}^{(t)}[j,:], B_{i+1}^{(t)}[j,:] \rangle - \sum_{s \in I} \langle W_{i+1}^{(t)}[j,:], \Delta^{(s)}[j,:] \rangle \\
= \langle W_{i+1}^{(t)}[j,:], B_{i+1}^{(t)}[j,:] \rangle - \langle W_{i+1}^{(0)}[j,:], B_{i+1}^{(0)}[j,:] \rangle + 0
\end{aligned}
\]
where $\Delta(s)$ indicates entries of $W_{i+1}^{(s)}$ that are equal to zero and $I$ is the set of times this occurs. However, by definition, the corresponding inner-product is zero and the result follows. Note this sub-argument is uneeded for FA so we are done.
\end{proof}

%%%%%%%%%%%%%%%%%%%%%%%%%%%%%%%%%%%%%%%%%%%%%%%%%%%%%%%%%%%%%%%%%%%%%%%%%%%%%%%
%%%%%%%%%%%%%%%%%%%%%%%%%%%%%%%%%%%%%%%%%%%%%%%%%%%%%%%%%%%%%%%%%%%%%%%%%%%%%%%

\subsection{Proof of theorem \protect\ref{convergence_result}}
\label{convergence}

\convergence*

\begin{proof}

We handle each case separately.

\textbf{Case $\beta =1$:} With $(\alpha,\beta)$-alignment we can proceed directly and integrate the bound directly:
$$
\cfrac{d \mathcal{L}(\theta^{(t)})}{dt} \le -\alpha \cdot \mathcal{L}(\theta^{(t)}) $$
$$
\Rightarrow \int \cfrac{d \mathcal{L}}{\mathcal{L}} \le - \int_0^t \alpha \cdot ds $$
$$
\Rightarrow \log(\mathcal{L}(\theta^{(t)})) - \log(\mathcal{L}(\theta^{(0)})) \le - \alpha \cdot t .
$$
Isolating the terms we obtain our result:
$$
\mathcal{L}(\theta^{(t)}) \le \mathcal{L}(\theta^{(0)}) \cdot e^{-\alpha \cdot t}.
$$

\textbf{Case $\beta > 1$:} With $(\alpha,\beta)$-alignment we can proceed directly and integrate the bound directly:
$$
\cfrac{d \mathcal{L}(\theta^{(t)})}{dt} \le -\alpha \cdot \mathcal{L}(\theta^{(t)})^{\beta} $$
$$
\Rightarrow \int \cfrac{d \mathcal{L}}{\mathcal{L}^{\beta}} \le - \int_0^t \alpha \cdot ds $$
$$
\Rightarrow - \cfrac{\beta-1}{\mathcal{L}(\theta^{(t)})^{\beta-1}} + \cfrac{\beta-1}{\mathcal{L}(\theta^{(0)})^{\beta-1}} \le - \alpha \cdot t .
$$
Isolating the loss in the equality above we obtain our result:
$$
- \cfrac{\beta-1}{\mathcal{L}(\theta^{(t)})^{\beta-1}}  \le - \alpha \cdot t - \cfrac{\beta-1}{\mathcal{L}(\theta^{(0)})^{\beta-1}} $$
$$
\cfrac{\beta-1}{\alpha t + \cfrac{\beta-1}{\mathcal{L}(\theta^{(0)})^{\beta-1}}} \ge \mathcal{L}(\theta^{(t)})
$$
$$
\Rightarrow \mathcal{L}(\theta^{(t)}) \le \left( \cfrac{\beta-1}{\alpha \cdot t + \cfrac{\beta-1}{\mathcal{L}(\theta^{(0)})}} \right)^{\cfrac{1}{\beta-1}}
$$

\end{proof}

\subsection{Proof of Theorem \protect\ref{application_two_layer}}
\label{application}

\application*

\begin{proof}

Let's first recall the basic setting under consideration. Here are our data assumptions for convenience.

\basicdata*

\linedata*

We consider the following network parameterized by \(\theta = (W, w_2)\) with an exponential margin-loss for the training objective:
\[
f(x;\theta) = w_2^T \phi(x W_1) , \quad \mathcal{L}(\theta) := \mathbb{E}_{(x,y) \sim P}[\ell(f(x; \theta), y)] .
\]
We consider a leaky ReLU activation $\phi(z) = \max(c', z)$ where we suppose $c' > 0$. Under the feedback alignment flow which we have the following for the trajectory of the model parameters wherever the following are defined:
\[
\tilde \nabla_{W_1} \mathcal{L}(\theta) = \mathbb{E}_k[y_k \ell'_k x_k^T b_2^T \text{diag}[\phi'(x_k W_1)] ], \quad \quad \tilde \nabla_{w_2} \mathcal{L}(\theta) = \mathbb{E}[y_k \ell_k' \text{diag}[\phi'(x_k W_1)] W_1^T x_k^T] .
\]
In particular, the feedback weights \(b_2\) replace \(w_2\) which in general makes the trajectory distinct from the gradient flow. 

Our goal is to establish alignment dominance. Actually, since we already know the update rule for \(w_2\) under feedback alignment equals the gradient we will just bound the first-term because:
\[
\langle \nabla \mathcal{L}(\theta), \tilde \nabla \mathcal{L}(\theta) \rangle =  \langle \nabla_{W_1} \mathcal{L}(\theta), \tilde \nabla_{W_1} \mathcal{L} \rangle + \Vert \nabla_{w_2} \mathcal{L}(\theta) \Vert_2^2 \ge \langle \nabla_{W_1} \mathcal{L}(\theta), \tilde \nabla_{W_1} \mathcal{L}(\theta) \rangle .
\]
Once we establish $(\alpha,\beta )$-alignment dominance we will be able to guarantee that the evolution of \(\theta^{(t)} = -\mathcal{L}(\theta^{(t)})\)  decreases the loss. There is a minor regularity concern because (leaky) ReLU networks may not be differentiable everywhere. However, at these non-differentiable points we may take whatever directional derivative yeilds the worst bound. Formally, we have the following:
\[
\begin{aligned}
\text{max} \left \lbrace \tilde \partial (\mathcal{L} \circ \theta)'(t)  \right \rbrace  = -\text{min} \left \lbrace \lim_{i \to \infty} \langle \nabla \mathcal{L}(\theta_i), \tilde \nabla \mathcal{L}(\theta_i) \rangle : \mathcal{L} \text{ differentiable at } \theta_i  \land \lim_{i \to \infty} \theta_i \to \theta  \right \rbrace \\
\ge  -\text{min} \left \lbrace \lim_{i \to \infty}  \alpha \cdot (\mathcal{L}(\theta_i) - \mathcal{L}^*)^{\beta} : \mathcal{L} \text{ differentiable at } \theta_i  \land \lim_{i \to \infty} \theta_i \to \theta  \right \rbrace \ge \alpha \cdot (\mathcal{L}(\theta) - \mathcal{L}^*)^{\beta}
\end{aligned}
\]
where the last-step follows from the continuity of \(\mathcal{L}(\theta)\). The proof can be formalized further to indicate the derivative may be "set-valued" e.g. Clarke differential. This is not central to the argument, however.

With the approach outlined we can proceed directly to the main proof.

\textbf{Part 1 and 2:} So we have:
\[
\begin{aligned}
\langle \nabla_{W_1} \mathcal{L}(\theta), \tilde \nabla_{W_1} \mathcal{L}(\theta) \rangle = \mathbb{E}_{i,j} [ \langle y_i \ell'_i x_i^T w_2^T \text{diag}[\phi'(x_i W_1)] , y_j \ell'_j x_j^T b_2^T \text{diag}[\phi'(x_j W_1)] \rangle ] \\
= \mathbb{E}_{i,j} [ \text{Tr}(y_i \ell'_i x_i^T \underbrace{w_2^T \text{diag}[\phi'(x_i W_1) ] \text{diag}[\phi'(x_j W_1) ] b_2}_{w_2^T T_{ij} b_2} x_j \ell'_j y_j) ] .
\end{aligned}
\]
Our goal is to lower-bound the whole expression. Notice that the inner-term is a scalar so we will lower-bound \(w_2^T T_{ij} b_2 > c\) for some constant $c > 0$ and then extract it from the expectation. For (sign) feedback alignment we know $b_2 = \text{sign}(w_2)$ from Lemmas \ref{angle} and \ref{sign_angle}. Additionally, each entry of $T$ is lower-bounded by the leaky factor of the leaky ReLU activation $\phi(z) = \max(c' \cdot z, z)$ which we suppose is $c' > 0$. So we obtain:

$$w_2^T T_{ij} b_2 = \sum_k w_2[k] \cdot b_2[k] \cdot T_{kk} \ge \langle w_2 , b_2 \rangle \cdot c' .$$

To lower-bound $\langle w_2, b_2 \rangle$ we effectively need to show alignment. We can use our initialization plus Theorem \ref{conservation_law} to do this:
$$
\langle w_{L}^{(t)}[j], b_{L}[j] \rangle $$
$$= \frac{1}{2} \Vert w_{L-1}^{(t)}[j] \Vert_F^2 + (\langle w_{L}^{(0)}[j], b_{L} \rangle -  \frac{1}{2} \Vert w_{L-1}^{(0)}[j] \Vert_F^2)$$
$$
=  \frac{1}{2} \Vert w_{L-1}^{(t)}[j] \Vert_F^2 + (\Vert w_L^{(0)} \Vert_F^2 -  \frac{1}{2} \Vert w_{L-1}^{(0)}[j] \Vert_F^2)$$
$$
\ge \frac{1}{2} \Vert w_{L-1}^{(t)}[j] \Vert_F^2 > 0.
$$

The first equality follows from Theorem \ref{conservation_law}, the last two steps follow from our initialization assumption. Therefore, there is some constant $c > 0$ such that $w_2^T T_{ij} b_2 > c$ as desired. Combining our results we obtain:
\[
\begin{aligned}
\langle \nabla_{W_1} \mathcal{L}(\theta), \tilde \nabla_{W_1} \mathcal{L}(\theta) \rangle &\ge c \cdot \mathbb{E}_{i,j} [ \text{Tr}(y_i \ell'_i x_i^T x_j \ell'_j y_j) ] \\
&\ge c \cdot \mathbb{E}_{i,j} [ \ell'_i  \ell'_j \langle x_i y_i ,  x_j y_j \rangle ] \\
&\ge c \cdot \mathbb{E}_{i,j} [ \text{Tr}(y_i \ell'_i x_i^T x_j \ell'_j y_j) ] \\
&\ge c \cdot \mathbb{E}_{i,j} [ \ell_i  \ell_j \langle x_i y_i ,  x_j y_j \rangle ] .
\end{aligned}
\]
The last step follows because we are considering the exponential margin loss.

\textbf{Assume Orthogonal Separability:} By our assumption of orthogonal separability we know that \(\langle x_i y_i ,  x_j y_j \rangle > \gamma\) and we obtain:
\[
\langle \nabla_{W_1} \mathcal{L}(\theta), \tilde \nabla_{W_1} \mathcal{L}(\theta) \rangle 
\ge c \cdot \gamma \cdot \mathcal{L}(\theta)^2 \\
\Rightarrow \partial^{\circ} (\mathcal{L} \circ \theta)'(t) \le - c \cdot \gamma \cdot \mathcal{L}(\theta)^2 .
\]
This establishes $(c \cdot \gamma, 2)$-alignment dominance. 

\textbf{Assume Near Orthogonality:} The first part of the proof establishes that,
$$
\langle \nabla_{W_1} \mathcal{L}(\theta), \tilde \nabla_{W_1} \mathcal{L}(\theta) \rangle \ge c \cdot \mathbb{E}_{i,j}[\ell_i \ell_j \langle x_i y_i , x_j y_j \rangle] 
$$
where $c > 0$ is some data independent constant that only depends on activation and initialization. Our proof strategy is to lower bound the expectation in two stages. Because the pairs $(i,j)$ are sampled independently we can break the full expectation into two parts.

Let $j \in [n]$ indicate an arbitrary sample. By our assumption of near orthogonality we know that $\Vert x_i \Vert_2^2 \ge n \cdot (\gamma + \epsilon)$ where $\gamma := \max_{i \not = j} \vert \langle x_i,  x_j \rangle \vert$ and $\epsilon > 0$ so we obtain:
$$
\mathbb{E}_i[\ell_i \langle x_i y_i , x_j y_j \rangle] = \frac{1}{n} \sum_{i = 1}^n \ell_i \langle x_i y_i , x_j y_j \rangle 
= \frac{1}{n}\ell_j \langle x_j y_j, x_j y_j \rangle + \frac{1}{n} \sum_{i \not = j} \ell_i \langle x_i y_i , x_j y_j \rangle $$
$$
\ge \frac{1}{n} \ell_j \left[ n \cdot (\gamma + \epsilon) \right] + \frac{1}{n} \sum_{i \not = j} \ell_i \langle x_i y_i , x_j y_j \rangle 
\ge \ell_j \cdot (\gamma + \epsilon) + \frac{1}{n} \sum_{i \not = j} \ell_i \langle x_i y_i , x_j y_j \rangle   .
$$
The first inequality follows from our definition of near orthogonality. From the definition of $\gamma := \max_{i \not = j} |\langle x_i, x_j \rangle |$ we see that the inter-sample inner-product satisfies:
$$
\langle x_i, y_i, x_j, y_j \rangle \ge -| \langle x_i, x_j \rangle| \ge - \max_{i \not = j} |\langle x_i, x_j \rangle | := - \gamma .
$$
Therefore, we can obtain the following bound:
$$
\mathbb{E}_i[\ell_i \langle x_i y_i , x_j y_j \rangle] 
\ge \ell_j \cdot (\gamma + \epsilon) - \frac{1}{n} \sum_{i \not = j} \ell_i \gamma $$
$$
\ge  \ell_j \epsilon + \gamma \left( \ell_j- \frac{1}{n} \sum_{i = 1}^n \ell_i \right) .
$$
The last inequality follows because the exponential margin loss is non-negative which allows us to add back in the removed term when $i = j$. Now we consider the full expectation:
$$
\mathbb{E}_{i,j}[\ell_i \ell_j \langle x_i y_i , x_j y_j \rangle] $$
$$
\ge \frac{1}{n} \sum_{j = 1}^n \ell_j^2 \epsilon + \gamma \left( \frac{1}{n} \sum_{j = 1}^n \ell_j^2 - \left( \frac{1}{n} \sum_{j=1}^n \ell_j \right) \left( \frac{1}{n} \sum_{i = 1}^n \ell_i \right) \right) .
$$
We simplify to obtain the following:
$$
\mathbb{E}_{i,j}[\ell_i \ell_j \langle x_i y_i , x_j y_j \rangle] $$
$$
\ge \frac{1}{n} \sum_{j = 1}^n \ell_j^2 \cdot \epsilon + \frac{1}{n^2} \gamma \cdot \left( n \cdot \sum_{j = 1}^n \ell_j^2 - \left( \sum_{j = 1}^n \ell_j \right) \left( \sum_{i = 1}^n \ell_i \right) \right) $$
$$
= \frac{1}{n} \sum_{j = 1}^n \ell_j^2 \cdot \epsilon + \frac{1}{n^2} \gamma \cdot \left( n \cdot \sum_{j = 1}^n \ell_j^2 - \left( \sum_{j = 1}^n \ell_j \right)^2 \right) .
$$
In the first step we factor out $n$. In the second step we combine summands. Since $\ell_j \ge 0$ we can use the Cauchy-Schwartz inequality to show the second term is non-negative:
$$
\left( \sum_{j = 1}^n \ell_j \right)^2 = \left( \sum_{j = 1}^n 1 \cdot \ell_j \right)^2 = | \langle 1, \ell \rangle|^2 \le n \cdot \sum_{j = 1}^n \ell_j^2  $$
$$
\Rightarrow n \cdot \sum_{j = 1}^n \ell_j^2 - \left( \sum_{j = 1}^n \ell_j \right)^2 \ge 0 .
$$
This also implies the following bound on the first term:
$$
\frac{1}{n} \cdot \sum_{j = 1}^n \ell_j^2 \ge  \left(\frac{1}{n} \sum_{j = 1}^n \ell_j \right)^2 = \mathcal{L}(\theta)^2 .
$$
Putting the two bounds together we obtain:
$$
\mathbb{E}_{i,j}[\ell_i \ell_j \langle x_i y_i , x_j y_j \rangle] \ge \mathcal{L}(\theta)^2 \cdot \epsilon + 0 = \epsilon \cdot \mathcal{L}(\theta)^2.
$$
To finish, we substitute our results:
$$
\Rightarrow \langle \nabla_{W_1} \mathcal{L}(\theta), \tilde \nabla_{W_1} \mathcal{L}(\theta) \rangle 
\ge c \cdot \epsilon \cdot \mathcal{L}( \theta)^2 $$
$$
\Rightarrow \partial^{\circ} (\mathcal{L} \circ \theta)'(t) \le - c \cdot \epsilon \cdot \mathcal{L}(\theta)^2 .
$$
This establishes $(c \cdot \epsilon, 2)$-alignment dominance.  

\end{proof}

\end{document}